\documentclass[10pt]{article}
\usepackage{inputenc}
\usepackage[T1]{fontenc}
\usepackage{graphicx}
\usepackage{tikz}
\usepackage{amsmath,amssymb}
\usepackage{paralist}
\usepackage{enumerate}
\usepackage{cite}
\usepackage[hypcap]{caption}
\usepackage{csquotes}
\usepackage{floatrow}
\usepackage{framed}
\usepackage{enumerate}
\usepackage{listings}
\usepackage{csquotes}
\usepackage{times}
\usepackage{graphicx}
\usepackage{setspace}

\usepackage[english]{babel}
\usepackage{latexsym, amssymb, graphicx}
\usepackage{authblk}
\usepackage{amsthm}
\newtheorem{thm}{Theorem}

\title{Formal Verification of Autonomous Vehicle Platooning}
\author[1]{Maryam Kamali}
\author[1]{Louise A. Dennis}
\author[1]{Owen McAree}
\author[2]{Michael Fisher}
\author[2]{Sandor M. Veres}
\affil[1]{Department of Computer Science, University of Liverool, UK}
\affil[2]{Department of Mechanical Engineering, University of Sheffield, UK}

\begin{document}

\maketitle

\begin{abstract}
The coordination of multiple autonomous vehicles into convoys or
platoons is expected on our highways in the near future. However, before
such platoons can be deployed, the new autonomous behaviours of the
vehicles in these platoons must be certified. An appropriate
representation for vehicle platooning is as a multi-agent system in
which each agent captures the ``autonomous decisions'' carried out by
each vehicle. In order to ensure that these autonomous decision-making
agents in vehicle platoons never violate safety requirements, we use
formal verification. However, as the formal verification technique
used to verify the agent code does not scale to the full system and as
the global verification technique does not capture the essential
verification of autonomous behaviour, we use a combination of the two
approaches. This mixed strategy allows us to verify safety
requirements not only of a model of the system, but of the actual
agent code used to program the autonomous vehicles.
\end{abstract}

%%%%%%%%%%%%%%%%%%
%% Introduction

\section{Introduction}
\label{Section:Introduction}

While ``driverless cars'' regularly appear in the media, they are neither
``driverless'' nor fully autonomous. Legal constraints ensure that
there must always be a responsible human in the vehicle. However,
although fully autonomous road vehicles remain futuristic, the automotive
industry is working on what are variously called ``road trains'',
``car convoys'' or ``vehicle platoons''.  Here, each vehicle
autonomously follows the one in front of it on the road, with the
front vehicle in the platoon/convoy/train being driven manually. This
technology is being introduced by the automotive industry in order to
improve both the safety and efficiency of vehicles on very congested
roads~\cite{PlatoonLegislation}. It is especially useful if the vehicles are trucks/lorries and
if the road is a multi-lane highway.

In these platoons, each vehicle clearly needs to communicate with
others, at least with the one immediately in front and the one
immediately behind.  Vehicle-to-vehicle (\textsc{V2V}) communication
is used at a lower (continuous control system) level to adjust each
vehicle's position in the lanes and the spacing between the
vehicles. \textsc{V2V} is also used at higher levels, for example to
communicate joining requests, leaving requests, or commands dissolving
the platoon. So a traditional approach is to implement the software
for each vehicle in terms of hybrid (and hierarchical) control systems
and to analyse this using hybrid systems techniques.

However, as the behaviours and requirements of these automotive platoons
become more complex there is a move towards much greater autonomy
within each vehicle. Although the human in the vehicle is still
responsible, the autonomous control deals with much of the complex
negotiation to allow other vehicles to leave and join, etc. 
%% As autonomy increases, a typical problem is that traditional hybrid
%% system approaches do not allow us to easily analyse the high-level
%% autonomous behaviours. 
Traditional approaches involve hybrid automata~\cite{Henzinger1996} in
which the continuous aspects are encapsulated within discrete states,
while discrete behaviours are expressed as transitions between these
states. A drawback of combining discrete decision-making and
continuous control within a hybrid automaton is that it is difficult
to separate the two (high-level decision-making and continuous
control) concerns. In addition, the representation of the high-level
decision-making can become unnecessarily complex.

As is increasingly common within autonomous systems, we use a hybrid
autonomous systems architecture where not only is the discrete
decision-making component separated from the continuous control
system, but the behaviour of the discrete part is described in much
more detail. In particular, the \emph{agent} paradigm is
used~\cite{Wooldridge02:book}. This style of architecture, using the
agent paradigm, not only improves the system design from an
engineering perspective but also facilitates the system analysis and
verification. Indeed, we use this architecture for actually
implementing automotive platoons, and we here aim to analyse the
system by verification.

Safety certification is an inevitable concern in the development of
more autonomous road vehicles, and verifying the safety and
reliability of automotive platooning is currently one of the main
challenges faced by the automotive industry. The verification of such
systems is challenging due to their complex and hybrid nature.
Separating discrete and continuous concerns, as above, potentially
allows us to reason about the decision-making components in isolation
and ensure that no decision-making component ever deliberately chooses
an unsafe state. However, the use of the `agent' concept alone is not 
enough for our purposes, since this can still make its autonomous
decisions in an `opaque' way. In order to be able to reason about, and
formally verify, the choices the system makes, we need a
\emph{rational agent}~\cite{WooldridgeRao99:book}. This not only makes
decisions, but has explicit representations of the \emph{reasons} for
making them, allowing us to describe not only what the autonomous
system chooses to do, but \emph{why} it makes particular
choices~\cite{CACM13}.

 %Particularly, when the verification technique is model checking due to state space explosion problem. 

%In this paper, we address the analysis and verication challenge of automotive platoon, employing hybrid autonomous systems architecture where discrete decision-making components are seperated from continuous control system, using the agent paradigm. 

%Different technologies are involved in development of the safe and efficient of such systems. The purpose of this paper is to develop a safe vehicle platooning using hybrid autonomous systems architecture, i.e., seperating discrete decision-maiking components from continuous control system, using the agent paradigm. We identify possible hazards that might platoon go through and determine how our developed platoon resists againts such hazards. 

%Seperating discrete and continous concerns potentially provides a means to reason about the decision-making components and ensure that decision-making component never deliberately chooses an unsafe state. However, using a general agent is not enough for our reasoning purpose because it only encapsulates some computational components of a system and makes its own decisions in an opague way. In order to be able to reason about the agent choices, we need an agent which determines its beliefs, goals and intentions. Rational agents provide this feature by explicitly making decisions about what action to perform, given their current beliefs, goals and intentions.

The Belief-Desire-Intention (BDI) model is one of the most widely used
conceptual models not only for describing rational agents but for actually
implementing them~\cite{rao:92a}. A BDI-style agent is characterised
by its beliefs, desires and intentions: \emph{beliefs} represent the
agent's views about the world; \emph{desires} represent the objectives
to be accomplished; while \emph{intentions} are the set of tasks
currently undertaken by the agent to achieve its desires. A BDI-style
agent has a set of plans, determining how an agent acts based on its
beliefs and goals, and an event queue where events (perceptions from
the environment and internal subgoals) are stored.  In this paper, we
use the \textsc{Gwendolen} programming language~\cite{GWENDOLEN2008},
developed for verifiable BDI-style programming, to implement
agent-based decision-making for an automotive platoon. This captures
the high-level, autonomous decision-making within each vehicle.
%, see\figurename~\ref{fig:HAAintro}.

%\begin{figure}[h]
%\vspace{-3mm}
%\centering
%\renewcommand{\figurename}{Fig.}
%\includegraphics[scale=0.45]{pic/ETAPS1a.pdf}
%\caption{ Simplified Hybrid Agent Architecture for an Individual Vehicle.}
%\vspace*{-3mm}\label{fig:HAAintro}
%\end{figure}
%\vspace*{-2mm}
As part of safety certification, we need to verify the agent
decisions, especially in combination with the other vehicles. An
autonomous rational agent makes decisions about what actions to
perform, etc, based on the \emph{beliefs}, \emph{goals} and
\emph{intentions} that the agent holds at that time. We use a
model-checking approach to demonstrate that the rational agent always
behaves in line with the platoon requirements and never deliberately
chooses options that end up in unsafe states. We verify properties of
the rational agent code using the AJPF
model-checker~\cite{Dennis2012}, one of the very few model-checkers
able to cope with complex properties of BDI agents. Unfortunately,
there are two drawbacks to using AJPF: currently, AJPF does not
support verification of timed behaviours; and AJPF is resource heavy
and cannot be used to verify the whole system. Consequently, in this
paper, we use a combined methodology for the verification of
automotive platooning. To evaluate timing behaviour, we use a
timed-automata abstraction and verify the system using the Uppaal
model-checker; to evaluate autonomous decisions, we apply AJPF to the
individual agents together with an abstraction of the other
vehicles/agents. Furthermore, we describe how these two approaches to
modeling, i.e. BDI models and timed-automata, can be combined to
provide an appropriate basis for verifying the behaviour of both
individual agents and the whole system.

The remainder of the paper is organised as follows. In
Section~\ref{Section:Automotive platoon} the automotive platoon and
platoon requirements are presented. In
Section~\ref{Section:Agent-based Development of Automotive Platoon}
the hybrid agent architecture and the agent-based decision-making for
automotive platoon are described. In
Section~\ref{Section:Verification} the analysis and verification of
an automotive platoon is considered. Finally, in
Section~\ref{Section:Conclusions}, concluding remarks are provided and
future work is discussed.

% We verify different goals of our individual agent models using AJPF model checker~\cite{Dennis2012}. To check the automotive platoon as a multi-agent system, we develop the multi-agent model of platoon as time-automata and verify its properties using UPPAAl~\cite{Uppaal2004}.

%%%%%%%%%%%%%%%%%%
%\vspace{-0.3cm}

%%%%%%%%%%%%%%%%%%
%% Vehicle Platooning
%\vspace{-0.3cm}

\section{Automotive Platoons}
\label{Section:Automotive platoon}
%\vspace*{-.52em}

An automotive platoon, enabling road vehicles to travel as a group, is
led by a vehicle which is driven by a professional driver
\cite{Shladover07, konvio2008, statreProject}. The following vehicles,
i.e, members of the platoon, are controlled autonomously. These
vehicles, equipped with low-level longitudinal (controlling speed) and
lateral (controlling steering) control systems, travel in a platoon
with pre-defined gaps between them. In addition, \textsc{V2V}
communication also connects the vehicles at an agent level.
%% Controlling and maintaining a safe platoon involves both local
%% control and agent-level strategies.
The lead
vehicle, via its agent, effectively carries out coordination over the
platoon: setting parameters, creating certificates of joining and
leaving, etc. Each individual vehicle observes its environment and
follows incoming commands from the lead agent. In what follows, we outline
the set of high-level automotive platoon concepts and procedures
including how to join and leave a platoon~\cite{bergenhem2010}. In
addition, the initial requirements on these procedures for the
development of safe and reliable platooning are explained. From these
we derive the formal properties to be verified.
%\vspace*{-2.03ex}

\subsection{Joining the Platoon} 
A vehicle can join a platoon either at the end or in the middle with
different control strategies being used.  The joining procedure is as
follows:
%\vspace*{-1.03ex}

\begin{itemize}
\itemsep=0pt
\item a non-member vehicle sends a joining request to the platoon
  leader, expressing the intended position in the platoon;
\item if the vehicle has requested to join from the rear, the leader
  sends back an agreement provided the maximum platoon length has not
  been reached and the platoon is currently in normal operation;
\item if the vehicle requests to join in front of (for example)
  vehicle X and the maximum platoon length has not been reached, the
  leader sends an ``increase space'' command to vehicle X, and when
  the leader is informed that enough spacing has been created
  (approx. 17 metres), it sends back an agreement to the joining
  vehicle;
\item upon receipt of an agreement, the joining vehicle changes its
  lane (changing lane is a manual procedure which is 
  performed by a driver);
\item once the vehicle is in the correct lane, its automatic speed
  controller is enabled and it approaches the preceding vehicle;
\item when the vehicle is close enough to the preceding vehicle (less 
than 20 metres), its automatic steering controller is enabled and it 
sends an acknowledgement to the leader; and, finally
\item the leader sends a ``decrease space'' command to vehicle X, 
and when the leader is informed that spacing has been back to normal 
(approximately 5 metres), it replies to the acknowledgement.
\end{itemize}
In order to ensure a safe joining operation, the following requirements
should be preserved within the agent-based decision-making components of
automotive platoon.
%
%\vspace*{-1.03ex}

\begin{enumerate}
\itemsep=0pt
\item A vehicle must only initiate joining a platoon, i.e., changing lane, 
once it has received confirmation from the leader.
\item Before autonomous control is enabled, a joining vehicle must
  approach the preceding vehicle, in the correct lane.
\item Automatic steering controller must only be enabled once the joining vehicle
  is sufficiently close to the preceding vehicle.
\end{enumerate}

%\subsection{Maintaining the Platoon} 
%%
%Each individual vehicle maintains a predefined space from the
%preceding vehicle and periodically sends its status to the platoon
%leader. If the leader detects a minor error, i.e. \emph{recoverable}
%latency, it sends an ``increase space'' command to the following
%vehicles. However, if the latency is not recoverable and endangers the
%platoon safety, the platoon should dissolve.
%
%To maintain a safe platoon, the following identified requirements
%should be ensured by the agent-based decision-making components of
%automotive platoon.
%%
%\begin{enumerate}
%\itemsep=0pt
%\item A vehicle must obey increasing/decreasing spacing commands in a
%  timely fashion.
%\item An acknowledgement should be sent to the leader after a request
%  has been fulfilled.
%\item A vehicle should modulate its own inter-vehicle spacing
%  depending on the communication latency and sensor performance.
%\item All changes in spacing should be communicated to the leader.
%\item Upon receipt of a request to perform an emergency stop, a
%  vehicle should come to halt immediately.
%\end{enumerate}
%\vspace*{-3.03ex}

\subsection{Leaving the Platoon} 
A vehicle can request to leave platoon at any time. The leaving
procedure is:
%\vspace*{-1.03ex}

\begin{itemize}
\itemsep=0pt
\item a platoon member sends a leaving request to the leader and waits
  for authorisation;
\item upon receipt of `leave' authorisation, the vehicle increases its
  space from the preceding vehicle;
\item when maximum spacing has been achieved, the vehicle switches
  both its speed and steering controller to `manual' and changes its
  lane; and, finally
\item the vehicle sends an acknowledgement to the leader.
\end{itemize}
%
%\vspace*{-1.03ex}

\noindent The two following requirements are necessary in order to meet with the
agent-based decision-making components of automotive platoon.
%
%\vspace*{-1.03ex}
\begin{enumerate}
\itemsep=0pt
\item Except in emergency cases, a vehicle must not leave the platoon
  without authorisation from the leader.
\item When authorised to leave, autonomous control should not be
  disabled until the maximum allowable platoon spacing has been
  achieved.
\end{enumerate}

%\subsection{Dissolving the Platoon} 
%Whole or part of a platoon may be dissolved due to large
%(unrecoverable) latency. Upon detection of large communication latency
%from a follower, the leader should send a ``dissolve platoon'' command
%to the follower that has high latency as well as to its following
%vehicles. However, if the communication latency is initiated from the
%leader or from multiple vehicles, the leader should send a ``dissolve
%platoon'' command to \emph{all} following vehicles.
%
%To safely dissolve a platoon, a vehicle should increase its spacing to
%the maximum allowable platooning distance and hand control back to the
%driver.

%%%%%%%%%%%%%%%%%%
%\vspace{-0.3cm}

%%%%%%%%%%%%%%%%%%
%% Agent-based Development of Platooning
%\vspace{-0.3cm}

\section{Agent-based Development of Automotive Platoon}
\label{Section:Agent-based Development of Automotive Platoon}
%\vspace*{-.52em}
We employ a hybrid agent architecture based on~\cite{Dennis2010} for
each vehicle:
\begin{center}
\includegraphics[width=0.895\textwidth]{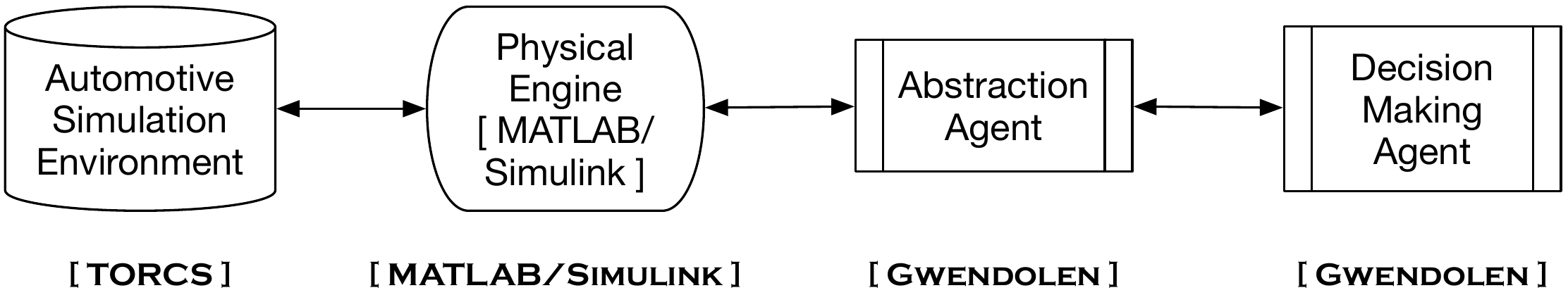}
\end{center}
%\vspace*{-1.03ex}

%
%% , shown in Fig.~\ref{fig:HAArch}, separating the
%% decision-making part from continuous feedback controllers.
%% 
%% \begin{figure}[h]
%% %\vspace*{-1.5em}
%% \centering
%% \renewcommand{\figurename}{Fig.}
%% \includegraphics[scale=0.45]{pic/HAArch-platoon.pdf}
%% \caption{ Detailed Architecture for an Individual Vehicle in the Platoon.}
%% \vspace*{-2mm}\label{fig:HAArch}
%% \end{figure}
%% %\vspace*{-4mm}
\noindent Real-time continuous control of the vehicle is managed by feedback
controllers, implemented in MATLAB, and observing the environment
through its sensory input. This is called the \emph{Physical
  Engine}. The Physical Engine, in turn, communicates with an
\emph{Abstraction Agent} that extracts discrete information from 
streams of continuous data and passes this on a \emph{Decision-Making Agent}. 
The Decision-Making Agent is a rational agent which directs the 
Physical Engine by passing it instructions through the Abstraction Agent. 
Instructions from the Decision-Making Agent to the Abstraction Agent are 
interpreted into meaningful instructions for Physical Engine.

To provide the complex environment necessary for effective simulation
and testing, we use an automotive simulator, TORCS~\cite{TORCS}, to
implement the environment component of the architecture. The Physical
Engine is implemented in MATLAB, while both Abstraction and
Decision-Making Agents are programmed in the \textsc{Gwendolen}
programming language. An interface between TORCS and MATLAB/Simulink
has been developed that provides a means to control vehicles from
MATLAB and Simulink.
%% A \emph{Java} environment interface
%% has also been developed to establish a connection between the
%% physical/continouos engines and abstraction/decision-making agents
%% through a socket. The interface also provides the shared beliefs which
%% are used by abstraction and decision-making agents.

%\vspace*{-1.03ex}
\begin{figure}[htbp]
\vspace*{-2mm}
\begin{small}\begin{center}\begin{tabular}{| c |  c | l |} 
			\hline perform goal & +!g [perform] & adding a new goal to perform some actions \\
			\hline achievement & +!g [achieve] & adding a new goal and continuously attempting the  \\
			goal & & plans associated with the goal until the agent has  \\
				& & acquired the belief \emph{g}\\
			\hline	believe & +b & adding a new belief \emph{b} \\
			\hline	disbelieve & -b & removing a belief \emph{b} \\
			\hline  on-hold believe & *b & waiting for belief \emph{b} to be true \\
			\hline  plan & trigger : guard 
			$ \leftarrow $ body & \emph{trigger} is the addition of a goal or a belief; \emph{guard} gives\\
			 & & conditions under which the plan can be executed; \emph{body} \\
			 & & is a stack of actions that should be performed \\
			
			\hline guard condition & B x & checking if belief x is perceivable \\
			\hline guard condition & G x & checking if goal x has been added \\
			\hline action & perf(x) & action x causing the execution of x 		\\	\hline

\end{tabular}
\end{center}
\end{small}
\caption{\textsc{Gwendolen}~\cite{GWENDOLEN2008} Syntax}\label{table:syntax}
\vspace*{-3mm}
\end{figure}
%\vspace*{-4mm}

Listing~\ref{list: joining-follower} shows some of the 
\textsc{GWENDOLEN} code from the Decision-Making Agent for the 
joining procedure for follower vehicle. %(The complete agent code is
%available in~\cite{CompletePlatoonAgent}.) 
The \textsc{Gwendolen}
syntax that is needed for this example is summarised in
Fig.~\ref{table:syntax}. Essentially, \textsc{Gwendolen} is an
extension of Prolog-style declarative programming, incorporating
explicit representations of goals, beliefs, and plans. For example,
the first plan in Listing~\ref{list: joining-follower} (Line 7-12)
denotes that once the follower agent sets a goal to join the platoon,
it sends a request to the leader and waits for an agreement belief to
become true. The changing lane plan in Listing~\ref{list:
  joining-follower} (Line 14-19) can be executed if and only if the
follower agent has the agreement belief. In each iteration, the
follower agent then selects plan based on its goals and beliefs.

A \textsc{Gwendolen} agent can also perform deductive reasoning on its
beliefs, expressed through its \emph{reasoning rules}. In the
`joining' scenario, the follower agent deduces that its goal to join
the platoon has been achieved if it believes \emph{all} the prescribed
joining steps have been performed. This is represented by the
\emph{platoon-ok} belief.

Essentially, the decision-making agent's activity proceeds in
sequence: the follower has a goal to successfully join the platoon; it
initiates changing lane, if it believes it has received an agreement
from the leader; and the follower achieves the joining goal if it
believes it is in right lane and the automatic speed and steering
controller are enabled. Changing lane is performed manually by a
driver and as long as the speed and steering controllers are not
switched to automatic, driver needs to control speed and steering.

%\vspace*{-2mm}
\begin{lstlisting}[language=Prolog, label=list: joining-follower, numbers=right, caption= A Follower Vehicle's code, captionpos= b, basicstyle=\scriptsize]

Reasoning Rules
joining(X, Y):- name(X), platoon-ok

Plans
+! joining(X, Y) [achieve]: {B name(X) , ~B join_agreement(X, Y)}
		<- +!speed_contr(0) [perform], +!steering_contr(0) [perform],
		   .send(leader, :tell, join_req(X, Y)), *join_agreement(X, Y);
	
+! joining(X, Y) [achieve]: {B name(X), B join_agreement(X, Y), 
		~B changed_lane, ~G set_spacing(Z) [achieve]}
		<- +!speed_contr(0) [perform], +!steering_contr(0) [perform],
		   perf(changing_lane(1)), *changed_lane;
			
+! joining(X, Y) [achieve]: {B name(X), B join_agreement(X, Y),
		B changed_lane, ~B speed_contr, ~ B steering_contr, 
		~B closing_enough, ~G set_spacing(Z) [achieve]}
		<- +!speed_contr(1) [perform], *joining_distance;
			
+! joining(X, Y) [achieve]: {B name(X), B join_agreement(X, Y),
		B changed_lane, B speed_contr, ~B steering_contr, 
		B closing_enough, ~G set_spacing(Z) [achieve]}
		<- +!steering_contr(1) [perform]; 			

+! joining(X, Y) [achieve]: {B name(X), B join_agreement(X, Y),
		B changed_lane, ~B speed_cont, ~B steering_contr, 
		B closing_enough, ~G set_spacing(Z) [achieve]}
		<- +!speed_contr(1) [perform], +!steering_contr(1) [perform]; 			
			
+! joining(X, Y) [achieve]: {B name(X), B join_agreement(X, Y),
		B changed_lane, B speed_contr, B steering_contr, 
		B closing_enough, ~B platoon_m, ~G set_spacing(Z) [achieve]}
		<- .send(leader,:tell,message(X,joined_succ),*platoon_m, platoon-ok;										
\end{lstlisting}
%\vspace*{-5mm}

Joining the platoon is described in \textsc{Gwendolen} as an
`achievement' goal, meaning that the agent continuously attempts the
plans given in Listing~\ref{list: joining-follower} until it believes
that \emph{platoon-ok} is true. This belief, \emph{platoon-ok},
denotes that not only is the vehicle in the right lane but that its
automatic controllers are enabled. It also determines that the leader
has received an acknowledgement from the vehicle, confirming it has
successfully joined the platoon. Subsequently, the agent deduces from
its reasoning rule (Line 2) that the joining goal has indeed been achieved.
%\vspace*{-5mm}

%%%%%%%%%%%%%%%%%%
%\vspace{-0.3cm}

%%%%%%%%%%%%%%%%%%
%% Verification
%\vspace{-0.3cm}

\section{Verification}
\label{Section:Verification}
%\vspace*{-.52em}

\subsection{Verification Methodology}
\label{subSection:Verification Methodology}
\noindent We can visualise the overall system as:
\begin{center}
\includegraphics[width=\textwidth]{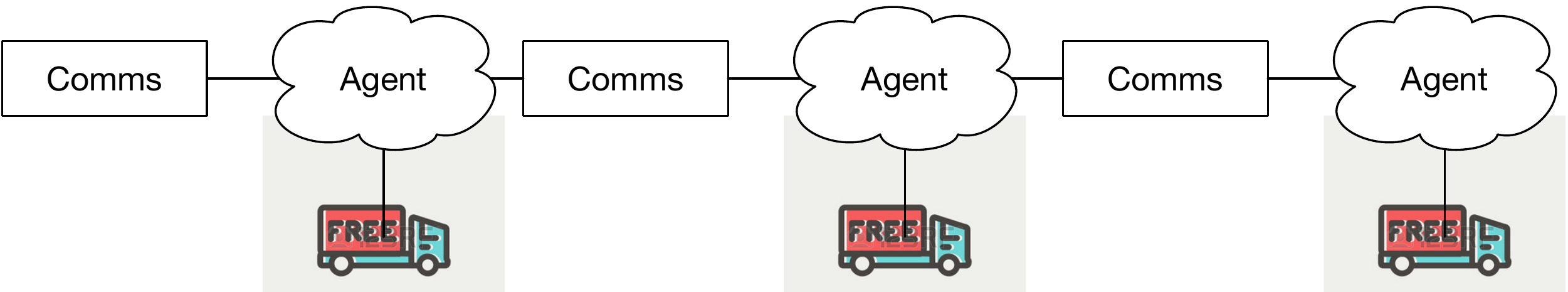}
\end{center}
%
%\vspace*{-0.6ex}
The agent is a \textsc{Gwendolen} program, the $Comms$ component is a simple transfer
protocol, and the vehicle represents the particular vehicular system
that we interact with. This is typically an automotive control system
together with environmental interactions, and we have validated this
both in simulation (using the TORCS automotive simulation) and in
physical vehicles (using Jaguar outdoor rover vehicles).

\begin{paragraph}{Limits to Modelling/Verification.}
We are not going to formally verify the vehicular control systems, and
leave this to standard mathematical (usually analytic) techniques from
the Control Systems field. These control components, for example
involving following a prescribed path, avoiding local obstacles,
keeping distance from object, etc, are well-established and standard.
Instead, we will verify the autonomous \emph{decisions} the vehicles
make, captured within each vehicle's `agent'~\cite{CACM13}. Each agent
represents the autonomous decision-maker within each vehicle and
corresponds, in part, to the human driver's decisions. These decisions
involve deciding \emph{where to go}, \emph{when to turn}, \emph{when
  to stop}, \emph{what to do in unexpected situations}, etc. In the
case of autonomous vehicle convoys/platoons, the agent's (and, hence,
the vehicle's) decisions concern \emph{when to join the convoy}, 
\emph{when to leave}, \emph{what to do in an emergency}, etc.
\end{paragraph}

So, we begin by abstracting from all the vehicle control systems and
environmental interactions, representing these by one (potentially
complex, depending on the vehicle/environment interactions)
automaton. We also use an automaton to describe the simple transfer
protocol that the vehicles use for their communication. In both these
cases we will use \emph{Timed Automata}~\cite{ACD93}. Simplified, our
architecture is:
\begin{center}
\includegraphics[width=\textwidth]{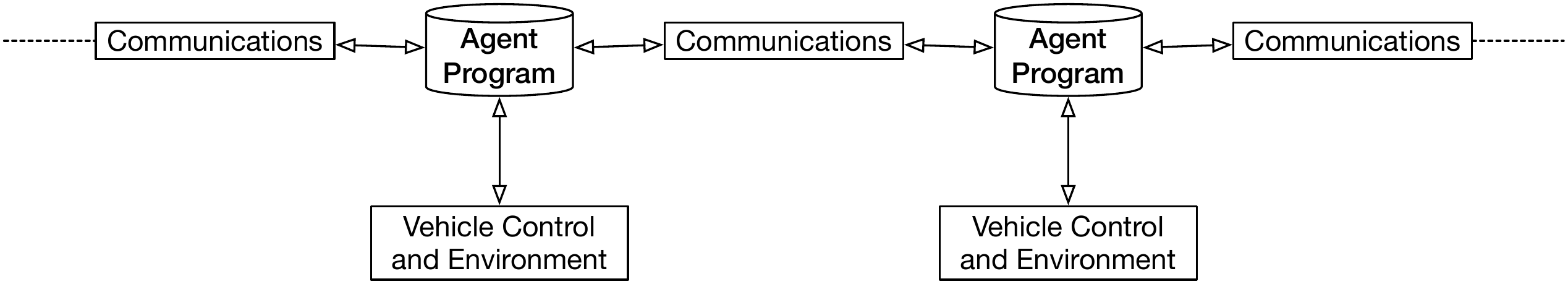}
\end{center}
which, at least in principle, leads to an overarching formal model:
\begin{center}
\includegraphics[width=\textwidth]{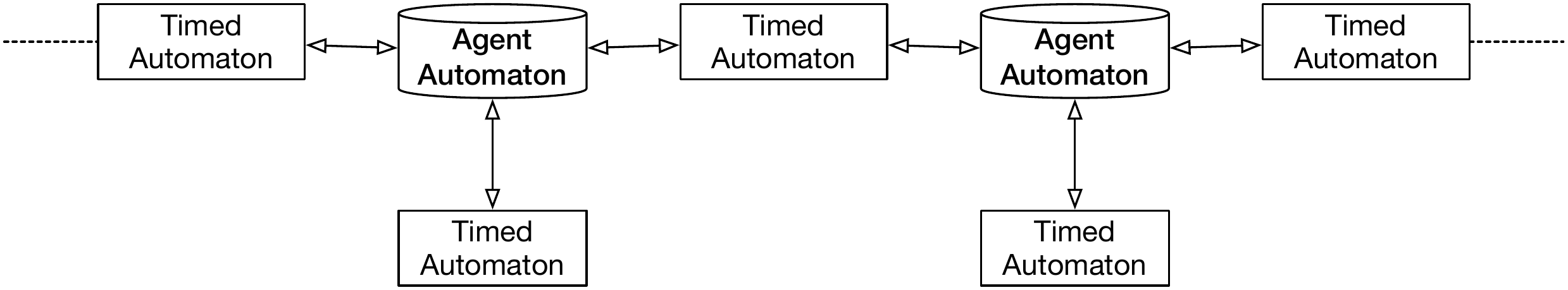}
\end{center}
In describing agent behaviour, the agent automaton comprises added
(modal) dimensions of (at least) \emph{belief} and
\emph{intention}. Thus, the formal structures that allow us to fully
represent all the system above are quite complex, combining timed
relations as well as relations for each of the belief and intention
dimensions~\cite{rao98:JLC,ACD93}. We will not describe this formal
model in detail but just note that it is a
\emph{fusion}~\cite{FingerG96,GKWZ03:book,Kurucz07:hbook} of timed and
BDI structures, $\langle L, A, C, \mathcal{E}, inv, R_B, R_I,
l\rangle$, where: $L$ is a finite set of locations; $A$ is a finite
set of actions; $C$ is a finite set of clocks; $\mathcal{E} \subseteq
L \times \Psi(C) \times A \times 2^{C} \times L$ is a set of (timed)
edges between locations; $inv: L \rightarrow \Psi(C)$ is a function
associating each location with some clock constraint in $\Psi(C)$;
$R_B: Ag\rightarrow (L \times L)$, where $Ag$ is the set of `agents'
and $R_B(a)$ provides the belief relation (corresponding to KD45 modal
logic) for agent $a$ between locations; $R_I: Ag\rightarrow (L \times L)$, where $R_B(c)$
provides the intention relation (KD) for agent $c$ between locations;
and $l : L \rightarrow 2^{AP}$ is a labelling function essentially
capturing those propositions true at each location ($AP$ is a set of
atomic propositions).

The logic then interpreted over such structures
combines~\cite{GKWZ03:book} the syntax of timed temporal logic, for
example $\diamondsuit^{\leq 5} \mathit{finish}$, and the syntax of
modal logics of belief, desire and intention, for example $B_x
\mathit{started}$ (i.e. ``agent $x$ believes that $\mathit{started}$
is true'').
\medskip

\noindent In principle, though very complex, we could provide all our
convoy requirements in such a logic, build structures of the above
form for our convoy implementation, and then develop a model-checking
approach for this combination~\cite{Konur2013}. However, there are
several reasons we choose to abstract and separate the timed/agent
strands, as follows.
\vspace*{-1em}

\begin{itemize}
\itemsep=0pt
\item For \emph{certification} it is important that we verify the
  \emph{actual} agent program used in each vehicle, not a derived
  model of this. Consequently, we utilise a \emph{program model
    checking}~\cite{Visser2003} approach to assess the correctness of
  each agent program. For this formal verification of the agent's
  autonomous decisions, we use AJPF~\cite{Dennis2012}, an extension of
  the Java PathFinder (JPF) program model checker~\cite{JPF:url} for
  \textsc{Gwendolen} that allows verification of belief/intention
  properties.
\item We do not have the detailed implementations of all the
  communications protocol, the vehicular control, and environmental
  interaction, and so use an abstract, formal model to describe these,
  rather than actual code.
\item JPF is an explicit-state program-checker and is relatively slow;
  AJPF builds a BDI programming layer on top of JPF and is at least an
  order of magnitude slower. Consequently, AJPF cannot realistically
  be used for verification of the whole system. In addition, as AJPF
  does not yet have real-time capabilities, then verifying timing
  aspects within AJPF is difficult.
\end{itemize}
While these appear problematic, there are several useful
simplifications in our context:
\begin{itemize}
\itemsep=0pt
\item When verifying autonomous behaviour, the formal
  verification we carry out concerns the interaction of beliefs,
  intentions, etc, \emph{within} each agent. These do not
  extend between agents and so, checking of beliefs, intentions, etc,
  can be localised within each agent.
\item In the requirements to be checked, the timed and BDI formulae
  are quite separate, i.e. $\diamondsuit^{\leq 5}
  \mathit{finish}\ \land\ B_x \mathit{started}$ but never
  $\diamondsuit^{\leq 5} B_x \mathit{started}$ or $B_x
  \diamondsuit^{\leq 5} \mathit{finish}$. As the overall logic is a
  fusion, and since there are no explicit timing constraints
  \emph{within} an agent program (agents have fast internal
  computation), then this allows us to deal with the dimensions
  separately.
\end{itemize}
So, given an overall system, $S$, over which we wish to check
$\varphi$, then we reduce $S\models\varphi$ to two problems:
\vspace*{-1em}

\begin{enumerate}
\itemsep=0pt
\item for each individual agent, $A_i$, within $S$, verify the agent
  properties from $\varphi$, i.e. $\varphi_a$, on the agent within an
  untimed environment (an over-approximation); and
\item verify the timing properties from $\varphi$ i.e. $\varphi_t$, on
  the whole system where the agent program is replaced by an untimed
  automaton describing solely its \emph{input-output} behaviour
  (abstracting from internal BDI reasoning).
\end{enumerate}
%\vspace*{-1em}

\noindent We explain both of these in more detail, before giving the
relevant theorems.

\paragraph{Timed Automaton $\longrightarrow$ Untimed Automaton.} This is 
achieved by the over-approximation as above and then allows us to
verify $V_i'\| Comms'\| A_i\models \varphi_a$ using AJPF:
\begin{center}
\includegraphics[width=\textwidth, height=2.5cm]{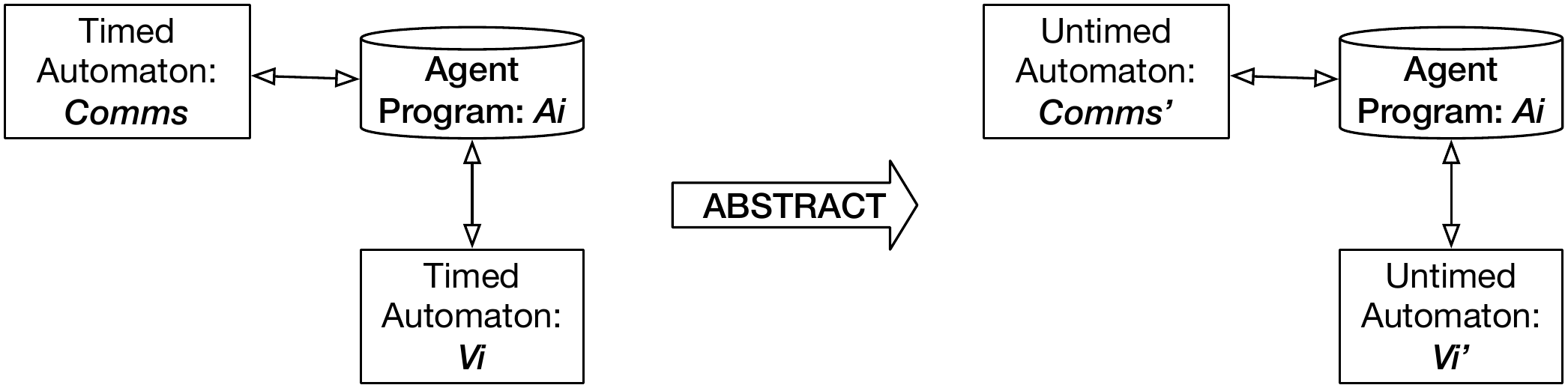}
\end{center}
Note that going from a timed automaton to an untimed one, for example
replacing behaviours such as $\diamondsuit^{\leq 3} receive$ by
$\diamondsuit receive$ provides this over-approximation.  Hence, for
example, $Comms'\models \varphi$ implies $Comms \models \varphi$, but
not necessarily the converse.

\paragraph{Agent Model $\longrightarrow$ Untimed Automaton.} This is
achieved by extracting a model of the agent program's behaviour, then
removing belief/intention aspects from this:
\begin{center}
\includegraphics[width=\textwidth, height=2.5cm]{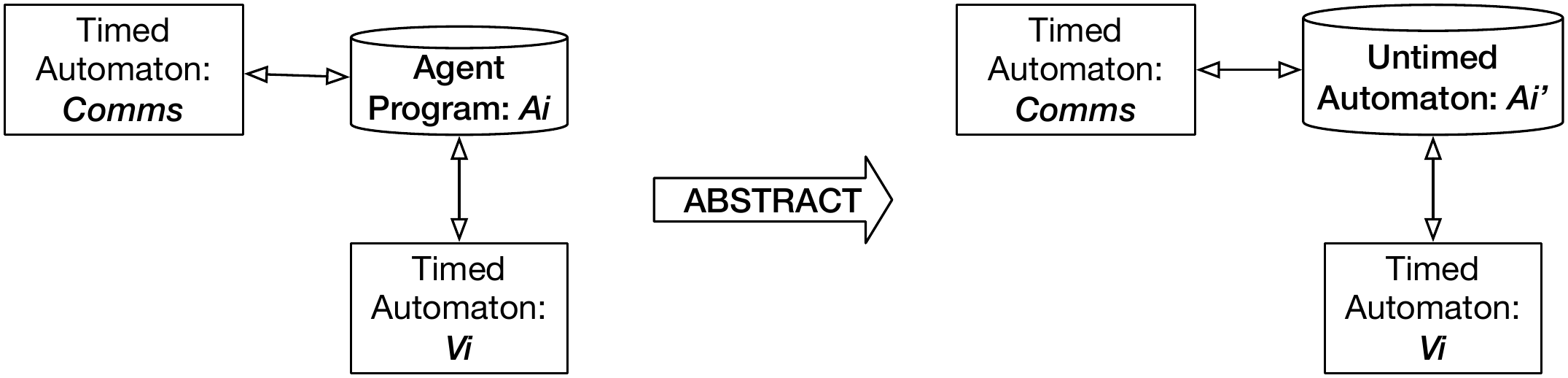}
\end{center}
Formal verification using Uppaal is then carried out on the whole
system with all agents abstracted in this way.  

We now prove important
properties of these abstractions. For simplicity, we assume that $S$
consists of just two agents/vehicles; this result can then easily be
generalised to greater numbers of agents/vehicles.

\begin{thm}
\label{thm1}
Let $S== V_1 \ \| \ A_1 \ \| \ Comms12 \ \| \ A_2 \ \| \ V_2$. If
\begin{itemize}
\itemsep=0pt
\item[a)] $V_1' \ \| \ A_1 \ \| \ Comms12'  \models \varphi_a$\quad and
\item[b)] $V_2' \ \| \ A_2 \ \| \ Comms12' \models \varphi_a$\quad and
\item[c)] $V_1 \ \| \ A_1' \ \| \ Comms12 \ \| \ A_2' \ \| \ V_2 \ \models \varphi_t$.
\end{itemize}
then $S\models\varphi_a\land\varphi_t$.
\end{thm}
\begin{proof}
Since $V_1'$ and $Comms12'$ are over-approximations, then 
\begin{center}
$V_1' \| A_1 \| Comms12' \models \varphi_a$ implies $V_1 \| A_1 \|
  Comms12 \models \varphi_a$.
\end{center}
Similarly, (b) gives us $V_2 \| A_2 \| Comms12 \models \varphi_a$. As
the agent properties in $\varphi_a$ are local, we can
compose these to give $V_1 \| A_1 \| Comms12 \| A_2 \|
V_2\models\varphi_a$ and so $S\models\varphi_a$.

By (c) we know that $V_1 \| A_1' \| Comms12 \| A_2' \| V_2 \models
\varphi_t$ yet, as $A_1$ and $A_2$ have no timed behaviour to begin
with, we know that $A_1'$ and $A_2'$ give us exactly the same timed
behaviours. Consequently, $V_1 \| A_1 \| Comms12 \| A_2 \| V_2 \models
\varphi_t$ and so $S\models\varphi_t$.

These two together give us $S\models\varphi_a\land\varphi_t$.
\end{proof}

\begin{thm}
If $V_1 \| A_1 \| Comms12 \| A_2 \| V_2\models \varphi_t$
then $V_1 \| A_1' \| Comms12 \| A_2' \| V_2 \models \varphi_t$.
\end{thm}

\begin{proof}
Since the timing behaviour of each $A_i$ is identical to each $A_i'$ then 
$V_1 \| A_1' \| Comms12 \| A_2' \| V_2$ and $V_1 \| A_1 \|
Comms12 \| A_2 \| V_2$ are equivalent.
\end{proof}

\subsection{Individual Agent Verification using AJPF}

To verify agent properties, we use the AJPF model checker on our
agent, written in the \textsc{Gwendolen} language, as above. For
instance, we verify that:

%\vspace*{-0.5em}
\begin{quote}
\emph{If a vehicle never believes it has received confirmation from the leader, then it never initiates joining to the platoon.} 
\end{quote}
%
%\vspace*{-0.5em}

This safety property corresponds to the first requirement of joining a
platoon, as given in Section~\ref{Section:Automotive platoon}, and can
be defined as:

%\vspace*{-1em}
\begin{small}
\begin{equation} \label{eq:prop0}
\begin{split}
\Box\ &{\tt \left( G_{\ f3} \ platoon\_ m\ (f3, f1)\ \right. }\\ 
&{\tt \rightarrow}
{\tt \ \lnot D_{f3} \ perf(changing\_ lane(1)) \ \mathit{W}\ }
{\tt \left. B_{f3} \ join\_ agr\ (f3, f1) \right) }
\end{split}
\end{equation}
\end{small}
%\vspace*{-1em}

\noindent Here $G_{x} \ y $ stands for a goal $y$ that agent $x$ tries
to achieve, $B_{x} \ z$ stands for a belief $z$ of agent $x$, and
$D_{x} \ k$ stands for an action $k$ that agent $x$ performs/does. The
standard LTL operators, such as $\Box$ meaning ``always in the
future'' and $\mathit{W}$ meaning ``unless'', are used. An instance of
the above , where the agent \emph{never} receives a join agreement,
is:

%\vspace*{-1em}
\begin{small}
\begin{equation} \label{eq:prop1}
\begin{split}
&{\tt  \Box\ \left( G_{\ f3} \ platoon\_ m\ (f3, f1)\  \& \right. } 
{\tt \lnot B_{\ f3} \ join\_ agr\ (f3, f1))}\\
&{\tt \rightarrow \ }
\Box {\tt \left. \lnot D_{\ f3} \ perf(changing\_ lane(1) \right)}\\
\end{split}
%\label{prp1}
\end{equation}
\end{small}

\noindent To be able to check such a property, incoming
perceptions/communications should be provided. We supply two automata:
$Comm'$, representing communication to/from the other agents; and
$V'_i$, representing vehicle responses to agent actions. Under this
configuration, we were able to carry out the agent verification in
around 12 hours.

We have verified a range of safety and liveness properties and we
provide some joining/leaving examples below. Note that the following
properties can also be similarly expressed in terms of the weak until
operator, $\mathit{W}$; however, we denote a particular instance of
these properties for the sake of brevity.
%\vspace*{-2mm}
\begin{quote}
\emph{If a vehicle ever sends a `join' request to the leader and
  eventually receives the join agreement and it is not already in the
  correct lane, it initiates `joining' the platoon by performing
  ``changing lane''.}
\end{quote}

%\vspace*{-4mm}
\begin{small}
\begin{equation} \label{eq:prop2}
\begin{split}
&{\tt \left(\ G_{\ f3} \ platoon\_ m\ (f3, f1)\ \& \ \lnot B_{\ f3} \ changed\_lane \ \&  \right. } \\
&{\tt \Box\ ItD_{\ f3} send (leader, tell, \ message(f3,1,f1)) } 
{\tt \left. \rightarrow   \lozenge \ B_{\ f3} \ join\_agr\ (f3, \  f1) \right)}\\
&{\tt \rightarrow}
{\tt \ \lozenge \  D_{\ f3} \ perf(changing\_ lane(1))}\\
\end{split}
\end{equation}
\end{small}
%\vspace*{-2mm}

\noindent Property~\ref{eq:prop2} is a liveness property ensuring that
eventually (using the LTL $\lozenge$ operator) the joining procedure
initiates the changing lane control system once its condition is
fulfilled. Similarly, we can verify other properties to show progress such as 
eventually the speed and steering controllers are switched to automatic if pre-conditions
hold. Other verified properties ensuring safe operation of the
platoon are as follows.
%Properties~\ref{eq:prop3}- \ref{eq:prop6} were also verified.
%\vspace*{-2mm}
\begin{quote}
\emph{If a vehicle never believes it has changed its lane, then it
  never switches to the automatic speed controller.}
\end{quote}
%\vspace*{-2em}

%\vspace*{-2mm}
\begin{small}
\begin{equation} \label{eq:prop3}
\begin{split}
&{\tt\Box\  \left(\ G_{\ f3} \ platoon\_ m\ (f3, f1)\ \& \right. }
{\tt \left. \lnot B_{\ f3} \ changed\_ lane \ \right) } \\
&{\tt \rightarrow \ }
\Box {\tt \ \lnot D_{\ f3} \ perf(speed\_ controller(1))}\\
\end{split}
\end{equation}
\end{small}
%\medskip
%\vspace*{-3mm}
%\vspace*{-2em}

\begin{quote}
\emph{If a vehicle never believes it has received a confirmation from
  the leader, then it never switches to the automatic speed
  controller.}
\end{quote}
%\vspace*{-2.5em}

\begin{small}
\begin{equation} \label{eq:prop4}
\begin{split}
&{\tt \Box\ \left(\ G_{\ f3} \ platoon\_ m\ (f3, f1)\ \& \right. } 
{\tt \left. \lnot B_{\ f3} \ join\_ agr\ (f3, f1)\ \right) }\\
&{\tt \rightarrow}
{\tt \ \Box \ \lnot D_{\ f3} \ perf(speed\_ controller(1))}\\
\end{split}
\end{equation}
\end{small}
%\vspace*{-2em}

\begin{quote}
\emph{If a vehicle never believes it is sufficiently close to the
  preceding vehicle, it never switches to the automatic steering
  controller.}
\end{quote}
%\vspace*{-2em}

%\vspace*{-2mm}
\begin{small}
\begin{equation} \label{eq:prop5}
\begin{split}
&{\tt \Box\ \left(\ G_{\ f3} \ platoon\_ m\ (f3, f1)\ \&\ \right.}
{\tt \left. \lnot B_{\ f3} \ joining\_ distance\ \right)} \\
&{\tt \rightarrow \ }
\Box {\tt \ \lnot D_{\ f3} \ perf(steering\_ controller(1))}\\
\end{split}
\end{equation}
\end{small}
%\medskip
%\vspace*{-3mm}
%\vspace*{-2em}

\begin{quote}
\emph{If a vehicle never believes it has received a confirmation from
  the leader to leave the platoon, i.e., increasing spacing has been 
  achieved, then it never disables its
  autonomous control.}
\end{quote}
%\vspace*{-2mm}
%
Note that the leader sends back the `leave' agreement to
\emph{follower3} if, and only if, it received an acknowledgement
from \emph{follower3} showing that spacing has been increased.
%\vspace*{-2em}

%\vspace*{-3mm}
\begin{small}
\begin{equation} \label{eq:prop6}
\begin{split}
&{\tt \Box\ \left(G_{\ f3} \ leave\_ platoon\ \&\ \right.} 
{\tt \left. \lnot B_{\ f3} \ leave\_ agr\ (f3) \right) }\\
&{\tt \rightarrow \ }
{\tt \Box\ \lnot D_{\ f3} \ perf(speed\_ controller(0))}\\
\end{split}
\end{equation}
\end{small}
%\vspace*{-1em}

\noindent It is important to recall that perceptions and
communications coming in to the agent are represented as internal
beliefs. Hence the proliferation of belief operators. The AJPF program
model checker explores all possible combinations of shared beliefs and
messages and so, even with the relatively low number of perceptions
above, the combinatorial explosion associated with exploring all
possibilities is very significant. Therefore, verifying the whole
multi-agent platooning system using AJPF is infeasible.
%For instance the model checker ran for 4 days to verify property~\ref{eq:prop1}.  

To verify the global properties of multi-agent platooning,
we use a complementary approach. We manually generate a model 
of the whole system as timed-automata and use the Uppaal model checker to establish the (timed) correctness of multi-agent platooning.
%A complementary approach is also taken into account to verify the global properties of the multi-agent platoon. For this, we model agent behaviour as time-automata and use Uppaal model checker to prove the correctness of multi-agent platoon. %Moreover, the global properties of automotive platoons are checked using Uppaal.
%
In the following, we review the relevant timed-automata and highlight some of the global safety properties of vehicle platooning that have been verified using Uppaal.
%\vspace*{-1em}

\subsection{Timed Automata Model of Automotive Platoons}
We model vehicle platooning in Uppaal as a parallel composition of
identical processes describing the behaviour of each individual
vehicle in the platoon along with an extra process describing the
behaviour of the platoon leader (the \emph{leader} automaton). Each of
these vehicle processes is a parallel composition of two timed
automata, \emph{vehicle} and \emph{agent}. The \emph{agent} automaton,
in turn, comprises both $Comms$ and $A'_i$ components, as given in
Section~\ref{subSection:Verification Methodology} . %The model is
%available at~\cite{CompletePlatoonAgent}.

The vehicle automaton supplies incoming perceptions for the agent
automaton. It describes the sensor models and action execution. The
vehicle automaton receives, and responds to, the action commands of
the corresponding agent through three pairs of binary channels
modelling \emph{change-lane}, \emph{set-space} and
\emph{join-distance} commands and responses. To model timing
behaviour, we define a clock for the vehicle automaton which models
the time assessments for ``changing lane'', ``setting space'' and
``joining distance'' actions. Based on engineers' study,
actions \emph{change-lane}, \emph{set-space} and \emph{join-distance}
take $20\pm 5$, $10\pm 5$ and $10\pm 5$ seconds, respectively.

%\vspace{-2mm}

The agent automaton models an \emph{abstracted} version of the
\textsc{Gwendolen} agent by excluding all internal computations of
the agent. The overall structure of an agent consists of 5 regions, shown
in Fig.~\ref{fig:agentAutomaton}. If the automaton is in the
\emph{IDLE} region, which consists of only one location, then the
agent does not perform any action at that moment. The regions
\emph{JOIN}, \emph{LEAVE}, \emph{SET-SPACE} and \emph{SW-STEERING}
represent the sequence of necessary communications with other agents
(and the vehicle) in order to achieve the agent's goals. If a vehicle is
part of the platoon it can leave the platoon or receive messages from
the leader to set spacing or switch steering controller. If a vehicle
is not part of the platoon, it can only join the platoon. Each agent
automaton contains two binary channels \emph{join-r[i][0]} and
\emph{leave-r[i][0]} to model the unicast sending of `join' and
`leave' requests to the leader and two binary channels
\emph{joined-suc[i][0]} and \emph{left-suc[i][0]} to model the unicast
sending of `join' and `leave' \emph{acknowledgements} to the
leader. These channels are used to model the message passing between
the following agents and the leader, modelled in decision-making agent 
(Section.~\ref{Section:Agent-based Development of Automotive Platoon}). 
Furthermore, each agent automaton also contains 
channels to send commands to its vehicle and receive acknowledgements from
its vehicle. The agent automaton has a clock \emph{process-time} that is used
to model the time consumption for achieving goals.
% For instance, we can evaluate the lower and upper bound for joining
%to a platoon.

%\vspace*{-1.5em}

\begin{figure}[hptb]
\centering
\renewcommand{\figurename}{Fig.}
\includegraphics[width=\textwidth]{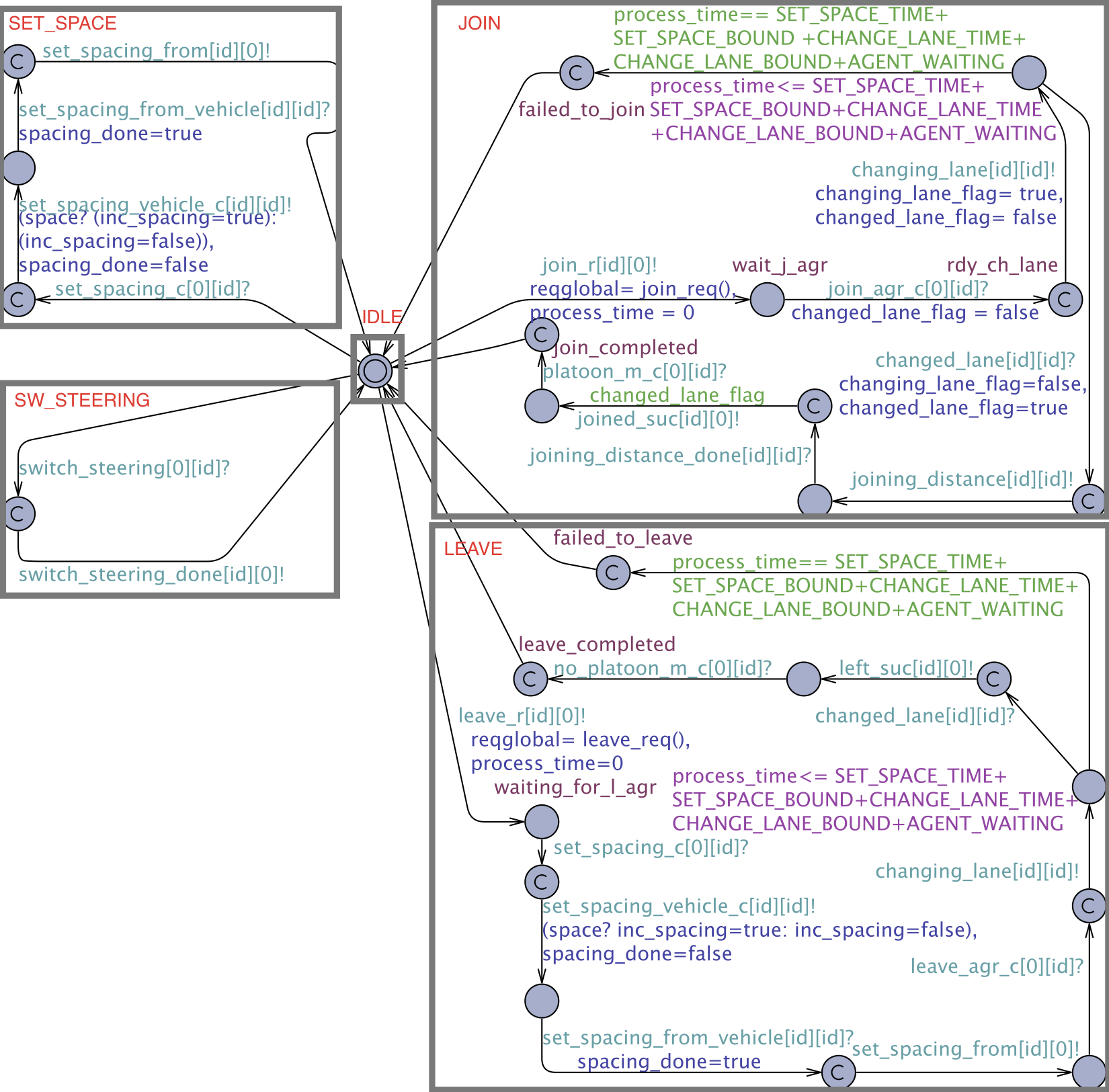}
\caption{ The Uppaal \emph{Agent} Automaton}
\vspace*{-2mm}\label{fig:agentAutomaton}
\end{figure}

Next, we define a leader automaton to model the external
behaviour of the leader agent (Fig.~\ref{fig:leaderAutomaton}), where
the coordination between agents is handled through unicast
synchronisation channels. Upon receipt of a joining request,
i.e.,~\emph{join-r[i][0]!}, it sends a ``set spacing'' command to the
preceding agent where the requested agent wants to be placed. The
leader sends a joining agreement, i.e.,~\emph{join-agr-c[0][i]?}, to
the requested agent, if it has successfully set spacing between the
two vehicles where the requested vehicle will be placed. Follower $i$
synchronises with the leader via~\emph{join-agr-c} channel. Then the
leader waits for an acknowledgement from the requested agent. It waits
for at most the upper bound time for setting space, changing lane and
getting close enough to the front vehicle. Upon receipt of the
acknowledgement, the leader sends a confirmation to the agent and a
``set spacing'' command to the preceding agent to decrease its space
with the front vehicle to complete the joining procedure. If it does
not receive the acknowledgement in time, it sends a ``set spacing''
command to the preceding agent to decrease its space and waits for a
spacing acknowledgement then goes back to the $idle$ location, ready
for the next request. The leader communicates with the agents through
synchronisation channels. It passes messages to the follower through
channels dedicated to the agreements, setting space and switching
steering controller. For simplicity, we assume the leader handles only
one request at any time.

%\vspace*{-1.5em}

\begin{figure}[h]
\vspace{-2mm}
\centering
\renewcommand{\figurename}{Fig.}
\includegraphics[width=\textwidth,height=80mm]{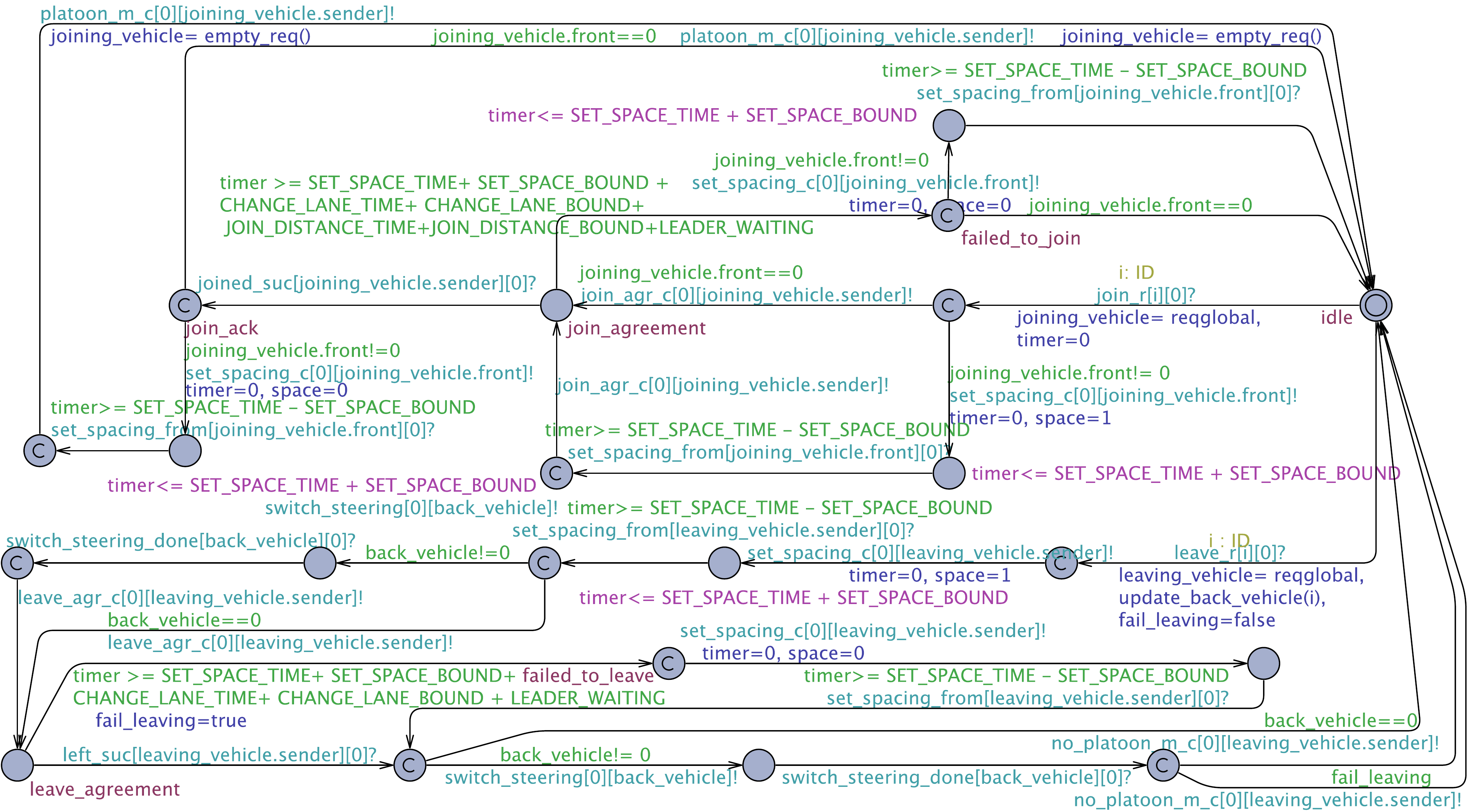}
\caption{ The \emph{leader} Automaton}
\vspace*{-2mm}\label{fig:leaderAutomaton}
\end{figure}
%\vspace*{-3em}

\subsection{Multi-agent Platooning Verification using Uppaal}
%
%\vspace*{-0.4em}

Now we have timed automata representations of the platoon, we can
carry out verification of their properties using Uppaal. For
simplicity, we analyse the global and timing properties of a
multi-agent platoon composed simply of a leader and three vehicles
(with three corresponding) agents. We assume vehicles can always set
spacing and joining distance in time, i.e., $10\pm 5$, but can fail to
change lane in time, i.e., less than $20 + 5$. In the following, we
first give examples of global properties involving the coordination
between the leader and the followers. Second, we evaluate timing
requirements: the safe lower and upper bounds for joining and leaving
activities. We observed that the verification of these properties took 
less than 3 seconds using Uppaal.

If an agent ever receives a joining agreement from the leader, then
the preceding agent has increased its space to its front agent. This
property is formulated for agent \emph{a3} as follows (\textbf{A}
represents ``on all paths''):

%\vspace*{-1.5em}
\begin{small}
\begin{equation} \label{eq:prop8}
\begin{split}
&{\tt \textbf{A} \Box\ \left( ( a3.rdy\_ch\_lane \ \textbf{\&\&} \ l. joining\_vehicle.front == 2)\  \right. }\\
&{\tt \left. \textbf{imply} \ (\ a2.incr\_spacing \ \textbf{\&\&} \ \ a2.spacing\_done) \right)}
\end{split}
\end{equation}
\end{small}
%\vspace*{-1em}

\noindent where \emph{a3} is the agent which is in the
$rdy\_ch\_lane$ location, i.e, the agent has received a joining
agreement, variable $joining\_vehicle.front$ indicates the
identification of the preceding agent, flag $a2.incr\_spacing$ models
that the preceding agent has received an ``increase space'' command
from the leader and, finally, flag $a2.spacing\_done$ models whether
agent \emph{a2} has successfully increased its space. We can also
verify this property for agents \emph{a2} and
\emph{a4}. Property~\ref{eq:prop8} is a safety requirement ensuring
that a vehicle initiates ``changing lane'' only if sufficient spacing
is provided.

The next property of interest is whether a joining request always ends
up increasing space of the preceding vehicle. To express this
property, we use the \emph{leads to} property form, written $\varphi
\rightsquigarrow \psi $. It states that whenever $\varphi$ is
satisfied, then eventually $\psi$ will be satisfied. Such properties
are written as $\varphi \dashrightarrow \psi$ in Uppaal. We verify the
property~(\ref{eq:prop9}) to show that whenever agent \emph{a3} is in
the \emph{wait\_ j\_ agr} location, i.e., has sent a joining
request, and agent \emph{a2} is the preceding vehicle in the platoon,
then eventually \emph{a3} will receive an increasing space command and
will perform the action.

%\vspace*{-1.5em}
\begin{small}
\begin{equation} \label{eq:prop9}
\begin{split}
&{\tt ( a3.wait\_j\_agr\ \textbf{\&\&} \ l. joining\_vehicle.front == 2)\  }\\
&{\tt \dashrightarrow (\ a2.incr\_spacing \ \textbf{\&\&} \ \ a2.spacing\_done) }
\end{split}
\end{equation}
\end{small}
%\vspace*{-1em}

\noindent To ensure that the spacing always decreases after a joining
procedure, i.e., platoon returns back to a normal state, we verify
that if ever the leader receives a joining request, it eventually
sends a decreasing space command to the preceding agent unless the
joined agent is the final one in the platoon.

%\vspace*{-1.5em}
\begin{small}
\begin{equation} \label{eq:prop10}
\begin{split}
&{\tt \textbf{A} \Box\ \left( ( a3.join\_completed \ \textbf{\&\&} \ l. joining\_vehicle.front == 2)\ \right. } \\
&{\tt  \left. \textbf{imply} \ (\ !a2.incr\_spacing \ \textbf{\&\&} \ \ a2.spacing\_done) \right) }
\end{split}
\end{equation}
\end{small}
%\vspace*{-1em}

\noindent Given the required time for a vehicle to carry out ``set
spacing'', ``joining distance'' and ``changing lane'' tasks, we are
interested in verifying if an agent accomplishes joining the
platoon within an expected interval: waiting time for agreement + changing
lane + joining distance + waiting time for leader confirmation,
represented in Property~\ref{eq:prop11}.

%\vspace*{-1.5em}
\begin{small}
\begin{equation} \label{eq:prop11}
\begin{split}
&{\tt \textbf{A} \Box\  \left( a2.join\_completed \ \ \textbf{imply} \right. }\\
 &{\tt \left. (a2.process\_time\geq50  \ \textbf{\&\&} \ a2.process\_time\leq90) \right) }
\end{split}
\end{equation}
\end{small}
%\vspace*{-1em}

\noindent Similarly, we check if an agent leaves a platoon within an
expected interval: waiting time for agreement + changing lane +
waiting time for leader confirmation. Waiting time for agreement is
equal to the time needed to set space and waiting time for leader
confirmation is zero because we assume switching
steering controllers is immediate.

%\vspace*{-1.5em}
\begin{small}
\begin{equation} \label{eq:prop12}
\begin{split}
&{\tt \textbf{A} \Box\  \left( a2.leave\_completed \ \ \textbf{imply} \right. }\\
 &{\tt \left. (a2.process\_time\geq30  \ \textbf{\&\&} \ a2.process\_time\leq50) \right)}
\end{split}
\end{equation}
\end{small}
%\vspace*{-2em}

%%%%%%%%%%%%%%%%%%
%\vspace{-0.3cm}

%%%%%%%%%%%%%%%%%%
%% Environment Modeling
%\vspace{-0.3cm}
%\input{sections/environment}
%%%%%%%%%%%%%%%%%%
%\vspace{-0.3cm}

\section{Concluding Remarks}
\label{Section:Conclusions}
%\vspace*{-.52em}

The verification of safety considerations for automotive platooning is
quite complex and difficult. There are several reasons for this.

%\vspace*{-0.5em}
\begin{itemize}
\itemsep=0pt
\item These are non-trivial hybrid autonomous systems, with each
  vehicle mixing feedback controllers and agent decision-making.
\item There is a strong requirement to verify the \emph{actual} code
  used in the implementation, rather than extracting a formal model of
  the program's behaviour --- this leads on to \emph{program}
  model-checking, which is resource intensive.
\item There are no other practical systems able to model check
  temporal and modal properties of complex BDI agents --- thus we are
  led to AJPF.
\item AJPF is \emph{very} resource intensive (as we have seen, 12
  hours for some agent properties) and cannot be practically used to
  verify whole system properties of automotive platooning.
\item Especially when interacting with real vehicles we need to verify
  timed properties.
\end{itemize}
Thus it is perhaps not surprising that such formal verification has
never been reported before. Safety verification of platooning in the contorl
level was investigated extensively~\cite{PathProject1998, Hilscher2013}. 
A combined verification approach for vehicle platooning is proposed 
in~\cite{colin2009} where the system behaviour is specified in CSP and B 
formal methods. A compositional verification approach for vehicle 
platooning is introduced in~\cite{Zaher2012} where feedback controllers and agent 
decision-making are mixed.

In order to address all of the above concerns, we have adopted a twin
strategy. We use AJPF to verify individual agent properties, given
realistic abstractions of environmental interactions. We then abstract
from the BDI code and produce an abstract agent automaton suitable for
use in Uppaal verification. This then allows us to formally verify
platoon requirements and safety considerations, a sample of which we
have included.

It must be emphasised that the agent code that we verify is actually
the code that controls the vehicle both in the TORCS simulation and in
the real vehicle that we are developing. Thus, as long as the environmental 
abstractions are correct, we can be sure of the decisions made by the agent.

\paragraph{Future Work.} There is clearly much future work to tackle. 
An obvious one is to continue efforts to improve the efficiency of AJPF.

Maintaining a safe platoon in case of \emph{recoverable} latency and
dissolving a platoon in the case of \emph{unrecoverable} latency are
two procedures that are not implemented in our verified agent code due
to a shortcoming of AJPF. Adding these two procedures to the agent
grows the system space to the extent that AJPF fails to verify any
property. Thus, we are investigating an agent abstraction at the level
of goals, beliefs and intentions in order to use AJPF for verification
of more complex agents.

Since we are concerned with certification of automotive platooning in
practice, we are aiming to extract a more comprehensive list of formal
properties from official platoon requirement documents. Related to
this, we are also in the process of porting the agent architecture on
to a real vehicle and so testing the platooning algorithms in
physical, as well as just simulation, contexts.

Finally, an important aspect of our two pronged strategy is to link
the models used in Uppaal to the programs that AJPF uses. In this
paper we generated the Uppaal models by hand, extracted temporal
formulae to capture their (non-timed) behaviour, and then verified
these temporal formulae on the agent code. This at least shows that
the timed automata we built correspond to the agent code execution. We
believe that all of this can be automated. In particular, we plan to
use the AJPF framework to explore the agent code executions and so
automatically build up the automaton that Uppaal can
use~\cite{TwoStage15}. This would give a much stronger form of
completeness and would improve efficiency.

%%%%%%%%%%%%%%%%%%
%\vspace{-0.3cm}

\newpage
\bibliographystyle{abbrv}
\bibliography{cav16} 

\end{document}